\documentclass{article}

\PassOptionsToPackage{numbers, sort}{natbib}

\usepackage{arxiv}
\usepackage{amsmath}
\usepackage{amsthm}

\usepackage[utf8]{inputenc} 
\usepackage[T1]{fontenc}    
\usepackage{url}            
\usepackage{booktabs}       
\usepackage{amsfonts}       
\usepackage{nicefrac}       
\usepackage{xcolor}         

\usepackage{microtype}
\usepackage{graphicx}
\usepackage{subcaption}
\usepackage{booktabs} 
\usepackage{xcolor}

\usepackage{hyperref}

\usepackage{pgfplots}
\pgfplotsset{compat = newest}
\usepackage{multirow}

\usepackage{amsmath}
\usepackage{amssymb}
\usepackage{mathrsfs}
\usepackage{mathtools}
\usepackage{amsthm}
\usepackage{algorithm}
\usepackage{algpseudocode}
\usepackage{listings}
\usepackage{makecell}
\usepackage{colortbl}
\usepackage{color}
\usepackage{wrapfig}
\usepackage{cancel}
\usepackage{soul,xcolor}
\usepackage{pifont}
\usepackage{tikz}
\usetikzlibrary{shapes, positioning, arrows.meta, calc, decorations.pathmorphing, quotes}

\usepackage[capitalize,noabbrev]{cleveref}

\theoremstyle{plain}
\newtheorem{theorem}{Theorem}[section]
\newtheorem{proposition}[theorem]{Proposition}
\newtheorem{lemma}[theorem]{Lemma}
\newtheorem{corollary}[theorem]{Corollary}
\theoremstyle{definition}
\newtheorem{definition}[theorem]{Definition}

\theoremstyle{remark}

\usepackage[textsize=tiny]{todonotes}

\newcommand{\alink}[1]{\href{#1}{paper-link}}

\usepackage{enumitem}
\usepackage{ amssymb }

\definecolor{codebg}{rgb}{0.95,0.95,0.95}
\definecolor{codeblue}{rgb}{0.13,0.13,1}
\definecolor{codegreen}{rgb}{0,0.5,0}
\definecolor{codegray}{rgb}{0.5,0.5,0.5}
\definecolor{codered}{rgb}{0.7,0,0}

\definecolor{citecolor}{HTML}{0071BC}
\definecolor{linkcolor}{HTML}{ED1C24}
\hypersetup{colorlinks=true, linkcolor=linkcolor, citecolor=citecolor,urlcolor=black}

\usepackage{amsmath,amsfonts,bm}

\def\eqref#1{equation~\ref{#1}}

\def\1{\bm{1}}

\DeclareMathAlphabet{\mathsfit}{\encodingdefault}{\sfdefault}{m}{sl}
\SetMathAlphabet{\mathsfit}{bold}{\encodingdefault}{\sfdefault}{bx}{n}

\definecolor{citecolor}{HTML}{0071BC}
\definecolor{linkcolor}{HTML}{ED1C24}
\hypersetup{colorlinks=true, linkcolor=linkcolor, citecolor=citecolor,urlcolor=black}

\title{Information-Theoretic Perspectives on Optimizers}
\author{%
    \textbf{Zhiquan Tan$^1$
    \quad
    Weiran Huang$^2$}\\
    $^1$ Department of Mathematical Sciences, Tsinghua University\\
    $^2$ Qing Yuan Research Institute, SEIEE, Shanghai Jiao Tong University
}

\begin{document}

\maketitle


\begin{abstract}
The interplay of optimizers and architectures in neural networks is complicated and hard to understand why some optimizers work better on some specific architectures. In this paper, we find that the traditionally used sharpness metric does not fully explain the intricate interplay and introduces information-theoretic metrics called entropy gap to better help analyze. It is found that both sharpness and entropy gap affect the performance, including the optimization dynamic and generalization. We further use information-theoretic tools to understand a recently proposed optimizer called Lion and find ways to improve it.

\end{abstract}

\section{Introduction}

Optimizers are critical to the success of neural network training, dictating how quickly and effectively a model converges to a solution. While several optimizers, such as SGD, and Adam are widely used, understanding why certain optimizers outperform others on specific architectures remains a challenge. Traditional methods for analyzing optimizers, such as sharpness metrics based on the Hessian matrix, offer valuable insights into the local landscape around minima or saddle points but fail to fully capture the underlying optimization dynamics. 

In this paper, we extend the analysis of optimizer behavior by introducing the concept in information theory. The entropy gap of the hessian provides more details of the optimizer's trajectory during training, offering an additional lens through which to examine convergence behavior and generalization. Along with sharpness, the entropy gap captures more local curvature statistics along optimization steps, shedding light on how optimizers traverse the loss surface and escape from saddle points.

We then apply information-theoretic tools to analyze the Lion optimizer's update rule, a recent addition to the family of optimizers. The Lion optimizer has gained attention for its simplicity and effectiveness, yet its effectiveness is still not fully understood. By incorporating tools like mutual information into the analysis of Lion, we uncover new insights into its optimization dynamics. Specifically, we show that it can be seen as an effective update direction information compressor. We further use these tools to help improve Lion.

Through detailed experiments and theoretical analysis, we demonstrate that both sharpness and entropy gap significantly affect the performance of different optimizers. By examining the training dynamics of popular architectures like ResNet and ViT under different optimizers, we highlight the importance of considering both metrics when evaluating optimizer effectiveness. 

This work offers a deeper, more nuanced understanding of optimizer behavior, providing a new avenue for designing and analyzing optimization algorithms. By bridging the gap between information theory and optimization dynamics, we propose a more robust framework for optimizer evaluation and improvement, which may also be applied to a wide range of other topics. Our findings contribute to the growing body of knowledge in optimization theory, suggesting that careful consideration of information-theoretic tools can lead to more efficient and generalizable optimization strategies.


\section{The information dynamics along training}

\begin{table}[t!]
\centering
\small
\caption{Test Accuracy of CIFAR10 along training.}
\label{tab:att}
\resizebox{\textwidth}{!}{
\begin{tabular}{c|c|c|c|c|c|c}
\toprule
 \textbf{Configuration} & \textbf{0 \% Steps} & \textbf{1 \% Steps} & \textbf{25 \% Steps} & \textbf{50 \% Steps} & \textbf{75 \% Steps} & \textbf{100 \% Steps} \\ 
\midrule
 ResNet18 (SGD) & 9.98  & 56.15 & 86.0 & 87.94  & 92.53 & 95.21\\ 
 ResNet18 (Adam) & 9.98 & 59.4  & 87.8 & 91.92 & 92.7 & 93.53 \\
ViT (SGD) & 10.01  & 19.35 & 68.35 & 77.61 & 81.81 & 83.23\\ 
ViT (Adam) & 10.01  & 39.89  & 75.99 & 83.8 & 86.2 & 87.21 \\
\bottomrule
\end{tabular}
}
\end{table}

\begin{table}[t!]
\centering
\small
\caption{Ratio of positive hessian among blocks during training.}
\label{tab:att}
\resizebox{\textwidth}{!}{
\begin{tabular}{c|c|c|c|c|c|c}
\toprule
 \textbf{Configuration} & \textbf{0 \% Steps} & \textbf{1 \% Steps} & \textbf{25 \% Steps} & \textbf{50 \% Steps} & \textbf{75 \% Steps} & \textbf{100 \% Steps} \\ 
\midrule
 ResNet18 (SGD) & 0.5968  & 0.9516 & 1.00 & 0.9839  & 0.9839 & 0.9839\\ 
 ResNet18 (Adam) & 0.5968  & 0.7581  & 0.9355 & 0.9032 & 0.9516 & 0.9032 \\
ViT (SGD) & 0.5405  & 0.8378 & 0.6216 & 0.6351 & 0.6622 & 0.5405\\ 
ViT (Adam) & 0.5405  & 0.8649  & 0.3784 & 0.4054 & 0.4189 & 0.4189 \\
\bottomrule
\end{tabular}
}
\end{table}

\begin{table}[t!]
\centering
\small
\caption{Sum of absolute hessian trace during training.}
\label{tab:att}
\resizebox{\textwidth}{!}{
\begin{tabular}{c|c|c|c|c|c|c}
\toprule
 \textbf{Configuration} & \textbf{0 \% Steps} & \textbf{1 \% Steps} & \textbf{25 \% Steps} & \textbf{50 \% Steps} & \textbf{75 \% Steps} & \textbf{100 \% Steps} \\ 
\midrule
 ResNet18 (SGD) & 0.015  & 0.0686 & 0.2065 & 0.2507  & 0.4960 & 0.7500\\ 
 ResNet18 (Adam) & 0.015  & 0.2643  & 3.700 & 12.39 & 23.41 & 29.68 \\
ViT (SGD) & 0.0232  & 0.0173 & 0.0051 & 0.0052 & 0.0033 & 0.0042\\ 
ViT (Adam) & 0.0232  & 0.0218  & 0.0109 & 0.0100 & 0.0092 & 0.0074 \\
\bottomrule
\end{tabular}
}
\end{table}

In this paper, we consider the training dynamics of different optimizers on different architectures, which include ResNet and ViT under SGD and Adam. During training, we can see that ResNet has a better landscape compared to ViT, making SGD behaves better than Adam on ResNet but vice versa on ViT. However, the sharpness (trace of hessian) cannot fully explain this behavior, making us consider new metrics like the entropy gap.

\begin{definition}[$\alpha$-order matrix entropy] Suppose a positive semi-definite matrix $\mathbf{R} \in \mathbb{R}^{n \times n}$ which $\mathbf{R}(i,i)=1$ for every $i=1, \cdots, n$ and $\alpha>0$. The $\alpha$-order (R\'enyi) entropy for matrix $\mathbf{R}$ is defined as follows:
$$
\operatorname{H}_\alpha\left(\mathbf{R}\right)=\frac{1}{1-\alpha} \log \left[\operatorname{tr}\left(\left(\frac{1}{n} \mathbf{R} \right)^\alpha\right)\right],
$$
where $\mathbf{R}^{\alpha}$ is the matrix power.

The case of $\alpha=1$ recovers the von Neumann (matrix) entropy, i.e., 
$$
\operatorname{H}_1\left(\mathbf{R}\right)=-\operatorname{tr}\left(\frac{1}{n} \mathbf{R} \log \frac{1}{n} \mathbf{R} \right).
$$
\end{definition}

In this paper, if not stated otherwise, we will always use $\alpha=1$. And the entropy gap will be $\ln n - \operatorname{H}(H)$ of a p.s.d. matrix $H$ of size $n \times n$.

\begin{table}[t!]
\centering
\small
\caption{The entropy gap on the test dataset.}
\label{tab:att}
\resizebox{\textwidth}{!}{
\begin{tabular}{c|c|c|c|c|c|c}
\toprule
 \textbf{Configuration} & \textbf{0 \% Steps} & \textbf{1 \% Steps} & \textbf{25 \% Steps} & \textbf{50 \% Steps} & \textbf{75 \% Steps} & \textbf{100 \% Steps} \\ 
\midrule
 ResNet18 (SGD) & 0.55  & 0.54 & 0.51 & 0.53  & 0.53 & 0.57\\ 
 ResNet18 (Adam) & 0.55  & 0.56  & 0.94 & 1.31 & 1.48 & 1.48 \\
ViT (SGD) & 0.43  & 0.43 & 0.34 & 0.34  & 0.34 & 0.34\\ 
ViT (Adam) & 0.43  & 0.36  & 0.34 & 0.33 & 0.32 & 0.32 \\
\bottomrule
\end{tabular}
}
\end{table}

\subsection{Entropy and training dynamics}

From \citep{xie2022adaptive}, we know that during optimization displacement from a saddle point along eigen-direction $i$ is $\Delta \theta_i^2 \sim O(|\lambda_i|)$, where $\lambda_i$ is the eigenvalue of hessian. Then it is clear that two values matter for the escape from saddle point speed: 1. The sharpness (trace of hessian) quantifies the total speed of escaping 2. The entropy gap $\log n - \operatorname{H}_1(\mathbf{H})$ measures the ``uniformness'' of escaping. A small gap will ensure the network escapes from the saddle point well in all directions.

The hessian entropy has a close relationship with the condition number of hessian, which directly relates to the convergence speed of optimization.

\begin{theorem}\label{entropy condition number}
Suppose the Hessian matrix $\mathbf{H} \in \mathbb{R}^{n \times n}$ is positive definite and $k$ is the condition number of it, then the following bound holds: $\log n - \operatorname{H}_1(\mathbf{H}) \leq \frac{k}{k+n-1} \log k - \log \frac{n+k-1}{n}  \leq \log k$.
\end{theorem}

From theorem \ref{entropy condition number}, it is clear a big entropy gap $\log n - \operatorname{H}_1(\mathbf{H})$ may force the network to have a big condition number. So a smaller gap is favorable for faster convergence of training.

\begin{proof}
We aim to prove that among all probability distributions satisfying the constraints:
\begin{equation}
    \sum_{i=1}^{n} a_i = 1, \quad \frac{a_{\max}}{a_{\min}} \leq k,
\end{equation}
the entropy
\begin{equation}
    H(\mathbf{a}) = -\sum_{i=1}^{n} a_i \ln a_i
\end{equation}
achieves its minimum value only for the following ``extreme'' distribution:
\begin{equation}
    \exists i_0: \quad a_{i_0} = M, \quad \forall j \neq i_0: \quad a_j = m,
\end{equation}
where
\begin{equation}
    \frac{M}{m} = k, \quad M + (n-1)m = 1.
\end{equation}

The proof consists of two parts:
\begin{itemize}
    \item {Prove that any ``non-extreme'' pair of elements can be adjusted to strictly decrease entropy.}
    \item {Conclude that the minimum entropy distribution must be the unique form where only one element attains the maximum value while all others take the same minimum value.}
\end{itemize}

\subsection*{1. Properties of the Function}
Consider the function:
\begin{equation}
    f(x) = -x \ln x, \quad x > 0.
\end{equation}
Computing the first and second derivatives:
\begin{equation}
    f'(x) = -\ln x - 1, \quad f''(x) = -\frac{1}{x} < 0.
\end{equation}
Since $f''(x) < 0$, the function $f(x)$ is strictly concave.

\subsection*{2. Adjustment Strategy}

Suppose at minimal point, $a_i > a_j$ and $\frac{a_i}{a_j} < k$. Define a small $\delta > 0$ such that:
\begin{equation}
    a_i' = a_i + \delta, \quad a_j' = a_j - \delta,
\end{equation}
ensuring normalization and that the ratio condition remains valid:
\begin{equation}
    \frac{a_i + \delta}{a_j - \delta} \leq k.
\end{equation}

Using Taylor expansion for small enough $\delta$,
\begin{equation}
    f(a_i+\delta) + f(a_j-\delta) < f(a_i) + f(a_j),
\end{equation}
which strictly decreases entropy, which yields a contradiction.

\subsection*{3. The minimum entropy distribution}

From the previous discussions, we know that the minimum will only be obtained when some of the points take only two values with the ratio of maximal to the minimum being $k$. Assume $s$ points take minimum value $m$ and $n-s$ points take maximum value $M=km$.

Then the entropy can be calculated as 
\begin{equation}
g(s) = \frac{s}{kn-(k-1)s} \log \frac{1}{kn-(k-1)s} - \frac{(n-s)k}{kn-(k-1)s} \log \frac{k}{kn-(k-1)s}.
\end{equation}
Then we can find that
\begin{equation}
g^{\prime}(s) = \frac{-(k-1)}{kn-(k-1)s}  + \log k \frac{n+(k-1)s}{(kn-(k-1)s)^2}.
\end{equation}
It is easy to see that
\begin{equation}
g^{\prime}(s) \leq \frac{-(k-1)}{kn-(k-1)s} +  (k-1) \frac{n+(k-1)s}{(kn-(k-1)s)^2} = -\frac{(k-1)^2n}{(kn-(k-1)s)^2} \leq 0.
\end{equation}
Therefore the minimum is obtained when $s=n-1$, the only possible form for the minimum entropy distribution is:
\begin{equation}
    M = \frac{k}{k+n-1}, \quad m = \frac{1}{k+n-1}.
\end{equation}    
Then calculating the entropy value will give the final bound.
\end{proof}

\subsection{Connection to network's generalization}

We will now discuss the effect the entropy gap plays in generalization.

\begin{theorem}
Let the likelihood function of the dataset \(\mathcal{D}\) be \(P(\mathcal{D}|\theta)\), with the corresponding negative log-likelihood function $\mathcal{L}(\theta) = -\log P(\mathcal{D}|\theta)$, and assume a Gaussian prior distribution $P(\theta) = \mathcal{N}(0, \sigma^2 I)$. Then, the posterior distribution can be approximated by
\[
P(\theta|\mathcal{D}) \approx \mathcal{N}\left(\theta^*, \left(H + \frac{1}{\sigma^2} I\right)^{-1}\right),
\] where $H$ is the hessian.  
\end{theorem}

\begin{proof}
By Bayes' theorem, the posterior distribution is
   \[
   P(\theta|\mathcal{D}) \propto P(\mathcal{D}|\theta) P(\theta) = \exp\left(-\mathcal{L}(\theta) - \frac{1}{2\sigma^2} \|\theta\|^2\right).
   \]
   Define the objective function
   \[
   J(\theta) = \mathcal{L}(\theta) + \frac{1}{2\sigma^2} \|\theta\|^2.
   \]

Let \(\theta^*\) be the local minimum of \(J(\theta)\), i.e.,
   \[
   \theta^* = \arg\min_{\theta} J(\theta),
   \]
   so that \(\nabla J(\theta^*) = 0\).

In the vicinity of \(\theta^*\), apply a second-order Taylor expansion to \(J(\theta)\):
   \[
   J(\theta) \approx J(\theta^*) + \frac{1}{2} (\theta - \theta^*)^T \nabla^2 J(\theta^*) (\theta - \theta^*).
   \]
   Here,
   \[
   \nabla^2 J(\theta^*) = \nabla^2 \mathcal{L}(\theta^*) + \frac{1}{\sigma^2} I = H + \frac{1}{\sigma^2} I.
   \]

Thus, we approximate
   \[
   P(\theta|\mathcal{D}) \propto \exp\left(-J(\theta)\right) \approx \exp\left(-J(\theta^*) - \frac{1}{2} (\theta - \theta^*)^T \left(H + \frac{1}{\sigma^2} I\right) (\theta - \theta^*) \right).
   \]
   Since \(J(\theta^*)\) does not depend on \(\theta\), it can be absorbed into a normalization constant.

Therefore, the posterior distribution is approximately
   \[
   P(\theta|\mathcal{D}) \approx \frac{1}{Z} \exp\left(-\frac{1}{2} (\theta - \theta^*)^T \left(H + \frac{1}{\sigma^2} I\right) (\theta - \theta^*) \right),
   \]
   which corresponds to a multivariate Gaussian distribution with mean \(\theta^*\) and covariance matrix \(\left(H + \frac{1}{\sigma^2} I\right)^{-1}\).
\end{proof}

We can now present a PAC-Bayes generalization bound which shows what metrics of the network mainly affect generalization.

\begin{enumerate}

    \item \textbf{Empirical Loss and Its Expectation:}  
    Let $\hat{L}(\theta)$ denote the empirical (training) loss at parameter $\theta$. For any distribution $Q$ over the parameter space, define
    \[
    \hat{L}(Q) = \mathbb{E}_{\theta \sim Q}[\hat{L}(\theta)].
    \]
    When the training set is sufficiently large and $Q$ is concentrated near $\theta^*$, we approximate
    \[
    \hat{L}(Q) \approx \hat{L}(\theta^*) \approx L(\theta^*).
    \]

    \item \textbf{PAC-Bayes Inequality:}  
    For a training set of size $n$, the PAC-Bayes framework guarantees that, with probability at least $1-\delta$ (for any $\delta\in(0,1)$),
    \[
    \mathbb{E}_{\theta\sim Q}[L(\theta)] \le \hat{L}(Q) + \sqrt{\frac{KL(Q\|P)+\ln\frac{1}{\delta}}{2n}}.
    \]
\end{enumerate}

\begin{theorem}[Information-Theoretic PAC-Bayes Generalization Bound]
Assume that there exists a local minimizer $\theta^* \in \mathbb{R}^d$ of the true loss function $L(\theta)$ with positive definite Hessian $H$ and $\mu I \preccurlyeq H$.
Let the prior and posterior be defined by
\[
P = \mathcal{N}(0,\,\sigma^2 I), \quad Q = \mathcal{N}\Big(\theta^*,\, \left(H + \frac{1}{\sigma^2} I\right)^{-1}\Big),
\]
and assume that the empirical risk under $Q$ satisfies $\hat{L}(Q) \approx L(\theta^*)$. Then, for any $\delta \in (0,1)$ and for a training set of size $n$, with probability at least $1-\delta$, the following generalization bound holds:
\begin{align}
&\mathbb{E}_{\theta\sim Q}[L(\theta)] -  L(\theta^*) \nonumber \\
&\leq \sqrt{\frac{1}{2n} (\ln \frac{1}{\delta} + d \ln( \frac{\sigma^2 + \mu}{d}) + \frac{\|\theta^*\|^2}{\sigma^2} + d (\ln \operatorname{tr}(H) + (\ln d - \operatorname{H}(H) ) - \frac{1}{d}(\ln d - \operatorname{H}(H) )^2 ) }.
\end{align}
\end{theorem}

\begin{proof}

For two Gaussian distributions 
\[
Q = \mathcal{N}(\mu_Q,\Sigma_Q) \quad \text{and} \quad P = \mathcal{N}(\mu_P,\Sigma_P),
\]
the KL divergence is given by
\[
KL(Q\|P) = \frac{1}{2} \left[ \operatorname{tr}\big(\Sigma_P^{-1}\Sigma_Q\big) + (\mu_Q-\mu_P)^\top \Sigma_P^{-1} (\mu_Q-\mu_P) - d + \ln\frac{\det\Sigma_P}{\det\Sigma_Q} \right].
\]

In our setting, we have:
\[
\mu_Q = \theta^*,\quad \Sigma_Q = \left(H + \frac{1}{\sigma^2} I\right)^{-1},\quad \mu_P = 0,\quad \Sigma_P = \sigma^2 I.
\]
Thus:
\begin{align*}
&\operatorname{tr}\big(\Sigma_P^{-1}\Sigma_Q\big) = \frac{1}{\sigma^2}\operatorname{tr}\Big(\left(H + \frac{1}{\sigma^2} I\right)^{-1}\Big),\\[1mm]
&(\mu_Q-\mu_P)^\top \Sigma_P^{-1} (\mu_Q-\mu_P) = \frac{\|\theta^*\|^2}{\sigma^2},\\[1mm]
&\ln\frac{\det\Sigma_P}{\det\Sigma_Q} = \ln\Big((\sigma^2)^d\,\det\left(H + \frac{1}{\sigma^2} I\right)\Big) = d\ln\sigma^2 + \ln\det\left(H + \frac{1}{\sigma^2} I\right).
\end{align*}
Therefore, the KL divergence is
\[
KL(Q\|P) = \frac{1}{2}\left[\frac{1}{\sigma^2}\operatorname{tr}\Big(\left(H + \frac{1}{\sigma^2} I\right)^{-1}\Big) + \frac{\|\theta^*\|^2}{\sigma^2} - d + d\ln\sigma^2 + \ln\det\left(H + \frac{1}{\sigma^2} I\right)\right].
\]    
\end{proof}

\begin{lemma}
Let $\{a_i\}_{i=1}^n$ be a probability distribution (i.e., $a_i>0$ and $\sum_{i=1}^n a_i=1$). Define
\[
A := -\frac{1}{n}\sum_{i=1}^n\ln a_i,\qquad
B := -\sum_{i=1}^n a_i\ln a_i.
\]
Then,
\begin{equation}
A - B \ge C\ln\frac{a_{\max}}{a_{\min}},
\end{equation}   
where $C=\frac{1}{2}\sum_{i=1}^n |a_i - \frac{1}{n}|$.
\end{lemma}
\begin{proof}
Note that
\begin{equation}\label{eq:AHdiff}
A - B = \sum_{i=1}^n\left(\frac{1}{n}-a_i\right)(-\ln a_i).
\end{equation}
Since
\[
\sum_{i=1}^n\left(\frac{1}{n}-a_i\right)=0,
\]
we partition the index set into
\[
I := \{i: a_i<1/n\} \quad \text{and} \quad J := \{i: a_i>1/n\}.
\]
Thus, we can rewrite \eqref{eq:AHdiff} as
\[
A - B = \sum_{i\in I}\left(\frac{1}{n}-a_i\right)(-\ln a_i)
+\sum_{j\in J}\left(\frac{1}{n}-a_j\right)(-\ln a_j).
\]
Note that for $j\in J$ we have $\frac{1}{n}-a_j<0$, so that
\[
A-B = \sum_{i\in I}\left(\frac{1}{n}-a_i\right)(-\ln a_i)
-\sum_{j\in J}\left(a_j-\frac{1}{n}\right)(-\ln a_j).
\]

Since the function $-\ln x$ is monotone decreasing (because $\ln x$ is increasing), we have:
\[
\text{For } i\in I:\quad -\ln a_i \ge -\ln a_{\min},
\]
and
\[
\text{For } j\in J:\quad -\ln a_j \le -\ln a_{\max},
\]
where
\[
a_{\min} := \min\{a_1,\dots,a_n\} \quad \text{and} \quad a_{\max} := \max\{a_1,\dots,a_n\}.
\]
Thus, for every $i\in I$,
\[
\left(\frac{1}{n}-a_i\right)(-\ln a_i) \ge \left(\frac{1}{n}-a_i\right)(-\ln a_{\min}),
\]
and for every $j\in J$,
\[
\left(a_j-\frac{1}{n}\right)(-\ln a_j) \le \left(a_j-\frac{1}{n}\right)(-\ln a_{\max}).
\]

Let
\[
C := \sum_{i\in I}\left(\frac{1}{n}-a_i\right) = \sum_{j\in J}\left(a_j-\frac{1}{n}\right)=\frac{1}{2}\sum_{i=1}^n \left|a_i - \frac{1}{n}\right|,
\]
where the equality follows from the fact that $\sum_{i=1}^n (\frac{1}{n}-a_i)=0$. Then, combining the above estimates, we obtain
\[
A-B \ge C\Bigl[(-\ln a_{\min})-(-\ln a_{\max})\Bigr]
= C\ln\frac{a_{\max}}{a_{\min}}.
\]    
\end{proof}

By noticing $C = \frac{1}{2}\sum_{i=1}^n \left|a_i - \frac{1}{n}\right|$, define $u = \left(\frac{1}{n}, \frac{1}{n}, \dots, \frac{1}{n}\right)$ be the uniform distribution. We have the following relationship:
\[
\|a-u\|_{TV} = \frac{1}{2}\sum_{i=1}^n \left|a_i - \frac{1}{n}\right|,
\]
where TV is the total variation distance.

Then we can conclude that,
\begin{lemma}
$\ln n - H(a) \le n\,\|a-u\|_{TV}.$  
\end{lemma}

\begin{proof}
Recall that the Kullback-Leibler divergence between \(a\) and \(u\) is
\[
D(a\|u) = \sum_{i=1}^n a_i \ln\left(\frac{a_i}{1/n}\right) = \ln n - H(a).
\]

We wish to show that
\[
\ln n - H(a) \le n\,\|a-u\|_{TV}.
\]

Define 
\[
I = \{ i \in \{1,2,\dots,n\} : a_i \ge \frac{1}{n} \}.
\]
For indices \(i \notin I\), we have \(a_i < \frac{1}{n}\) and hence
\[
\ln\left(n\,a_i\right) < 0.
\]
Thus, we can bound the divergence by restricting the sum to indices in \(I\):
\[
D(a\|u) = \sum_{i=1}^n a_i \log\left(n\,a_i\right) \le \sum_{i\in I} a_i \log\left(n\,a_i\right).
\]

For each \(i\in I\), let
\[
\Delta_i = a_i - \frac{1}{n} \ge 0.
\]
Notice that from the fact that $\sum_i (a_i - \frac{1}{n}) = 0$,
\[
\sum_{i\in I} \Delta_i = \sum_{i\in I} \left(a_i - \frac{1}{n}\right) = \|a-u\|_{TV}.
\]
Since for \(i\in I\) we have
\[
a_i = \frac{1}{n} + \Delta_i \le \frac{1}{n} + \|a-u\|_{TV},
\]
multiplying by \(n\) gives
\[
n\,a_i \le 1 + n\,\|a-u\|_{TV}.
\]
Taking logarithms (and using the monotonicity of \(\log\)) yields
\[
\ln\left(n\,a_i\right) \le \ln\Bigl(1+n\,\|a-u\|_{TV}\Bigr).
\]

Substitute this bound back into the divergence:
\[
D(a\|u) \le \sum_{i\in I} a_i\, \ln\Bigl(1+n\,\|a-u\|_{TV}\Bigr)
= \ln\Bigl(1+n\,\|a-u\|_{TV}\Bigr) \sum_{i\in I} a_i.
\]
Since \(\sum_{i\in I} a_i \le 1\), we conclude that
\[
D(a\|u) \le \ln\Bigl(1+n\,\|a-u\|_{TV}\Bigr) \leq n\|a-u\|_{TV}.
\]
Recalling that \(D(a\|u) = \ln n - H(a)\), we obtain the desired bound.    
\end{proof}

\subsection{Deviation among blocks}

\begin{table}[t!]
\centering
\small
\caption{The deviation of entropy.}
\label{tab:att}
\resizebox{\textwidth}{!}{
\begin{tabular}{c|c|c|c|c|c|c}
\toprule
 \textbf{Configuration} & \textbf{0 \% Steps} & \textbf{1 \% Steps} & \textbf{25 \% Steps} & \textbf{50 \% Steps} & \textbf{75 \% Steps} & \textbf{100 \% Steps} \\ 
\midrule
 ResNet18 (SGD) & 0.17  & 0.09 & 0.24  & 0.27  & 0.26 & 0.26\\ 
 ResNet18 (Adam) & 0.17  & 0.18  & 0.71 & 1.02 & 1.14 & 1.16 \\
 ViT (SGD) & 0.20  & 0.18 & 0.08  & 0.07  & 0.06 & 0.06\\ 
 ViT (Adam) & 0.20  & 0.09  & 0.07 & 0.05 & 0.05 & 0.05 \\
\bottomrule
\end{tabular}
}
\end{table}

We find that for SGD and Adam, the better performance optimizer for each architecture has a smaller deviation among blocks, showing a more uniform structure, which also relates to the coefficient of variation. 

\begin{theorem}
Suppose $x_i$ ($i=1, \cdots,n$) are i.i.d. random samples taken from $\mathcal{N}(a, b^2)$, where $a>0$. Then $-\sum^n_{i=1} \frac{|x_i|}{\sum^n_{i=1} |x_i|} \log \frac{|x_i|}{\sum^n_{i=1} |x_i|} \sim \log n - O(\frac{b^2}{a^2})$, where $\frac{b}{a}$ is often called the coefficient of variation. 
\end{theorem}
\begin{proof}
Note $-\sum^n_{i=1} \frac{|x_i|}{\sum^n_{i=1} |x_i|} \log \frac{|x_i|}{\sum^n_{i=1} |x_i|} = \frac{- \sum^n_{i=1} |x_i| \log |x_i|}{\sum^n_{i=1} |x_i|} + \log \sum^n_{i=1} |x_i| \sim \log na - \frac{\sum^n_{i=1} |x_i| \log |x_i|}{na}.$ From Taylor expansion, we know $\log X \approx \log a + \frac{X-a}{a} - \frac{(X-a)^2}{2 a^2}.$ Then it is straightforward to calculate $\mathbb{E} [X \log X] \approx a \log a + O(\frac{b^2}{a})$ by noticing the fact that $\mathbb{E}X =a$, $\mathbb{E}X^2 =a^2 + b^2$ and $\mathbb{E}X^3 =a^3+3ab^2$. Then the conclusion follows from the law of large numbers.
\end{proof}

\section{Understanding the update rule}

\begin{algorithm}
\caption{Lion Optimizer}
\begin{algorithmic}[1]
\State \textbf{Input:} Parameters $\theta_{t-1}$, momentum $m_0=0$, gradients $\nabla \theta_{t}$, learning rate $\eta_t$, momenta $\beta_1$, $\beta_2$
\For{$t = 1 \dots T$}
\State $U_t \gets \beta_2 m_{t-1} + (1 - \beta_2) \nabla \theta_{t}$
\State $\theta_t \gets \theta_{t-1} - \eta_t \cdot \text{sign}(U_t)$ 
\State $m_t \gets \beta_1 m_{t-1} + (1 - \beta_1)\nabla \theta_{t}$
\EndFor
\end{algorithmic}
\end{algorithm}

As the previous sections discuss the information dynamics of training of some optimizers, it is interested to see that can information theory be used to analyze the update rule and improve it. In this section, we will focus on a new optimizer called Lion.

To model this optimizer's behavior, we consider each component of its $U_t$ encodes a ``true'' update direction $Y$ and transform the problem as follows.
\begin{definition}\label{def:setup}
A binary classification problem satisfies:
\begin{itemize}
\item \textbf{Label space}: \( Y \in \{+1, -1\} \) with uniform prior:
  \[
  \mathbb{P}(Y=+1) = \mathbb{P}(Y=-1) = \frac{1}{2}.
  \]
\item \textbf{Observation model}: Given \( Y = y \), the observation \( X \) follows:
  \[
  X \mid Y=+1 \sim \mathcal{N}(\mu, \sigma^2), \quad X \mid Y=-1 \sim \mathcal{N}(-\mu, \sigma^2),
  \]
  where \( \mu \in \mathbb{R} \), \( \sigma > 0 \).
\end{itemize}
\end{definition}

\begin{theorem}[Bayes Optimal Classifier]\label{thm:bayes-rule}
Under Definition~\ref{def:setup}, the Bayes optimal classifier is:
\[
f^*(X) = \text{\rm sign}(X),
\]
with the decision rule:
\[
f^*(X) = 
\begin{cases}
+1 & \text{if } X \geq 0, \\
-1 & \text{if } X < 0.
\end{cases}
\]
\end{theorem}

\begin{proof}
\textbf{Step 1} (Posterior Calculation). By Bayes' theorem:
\[
\mathbb{P}(Y=+1 \mid X) = \frac{p(X \mid Y=+1)\mathbb{P}(Y=+1)}{p(X)}.
\]
Using total probability:
\[
p(X) = \frac{1}{2}\left[p(X \mid Y=+1) + p(X \mid Y=-1)\right].
\]

\textbf{Step 2} (Explicit Posterior). Substitute Gaussian densities:
\[
\frac{p(X \mid Y=+1)}{p(X \mid Y=-1)} = \exp\left(\frac{2\mu X}{\sigma^2}\right) \implies \mathbb{P}(Y=+1 \mid X) = \frac{1}{1 + \exp(-2\mu X/\sigma^2)}.
\]

\textbf{Step 3} (Decision Boundary). The Bayes rule selects \( y \) maximizing \( \mathbb{P}(Y=y \mid X) \):
\[
\mathbb{P}(Y=+1 \mid X) \geq 0.5 \iff X \geq 0 \quad \text{(symmetry for \( Y=-1 \))}.
\]
Thus, \( f^*(X) = \text{\rm sign}(X) \).
\end{proof}

\begin{theorem}[Minimum Error Rate Property]\label{thm:error}
For any classifier \( f: \mathbb{R} \to \{+1, -1\} \):
\[
R(f^*) \leq R(f),
\]
where \( R(f) = \mathbb{P}(f(X) \neq Y) \), and the Bayes error rate is:
\[
R(f^*) = \mathbb{E}_X\left[\min\left\{\mathbb{P}(Y=+1 \mid X), \mathbb{P}(Y=-1 \mid X)\right\}\right].
\]
\end{theorem}

\begin{proof}
Decompose the error rate:
\[
R(f) = \mathbb{E}_X\left[\mathbb{P}(Y \neq f(X) \mid X)\right].
\]
For each \( X \), the Bayes classifier achieves:
\[
\mathbb{P}(Y \neq f^*(X) \mid X) = \min\left\{\mathbb{P}(Y=+1 \mid X), \mathbb{P}(Y=-1 \mid X)\right\},
\]
while any suboptimal \( f \) satisfies:
\[
\mathbb{P}(Y \neq f(X) \mid X) \geq \min\left\{\mathbb{P}(Y=+1 \mid X), \mathbb{P}(Y=-1 \mid X)\right\}.
\]
Taking expectations preserves the inequality.
\end{proof}

\begin{corollary}[Error Rate Upper Bound]\label{cor:bound}
The Bayes error rate satisfies:
\[
R(f^*) \leq 0.5,
\]
with equality iff \( \mu = 0 \).
\end{corollary}

\begin{proof}
For all \( X \):
\[
\max\left\{\mathbb{P}(Y=+1 \mid X), \mathbb{P}(Y=-1 \mid X)\right\} \geq 0.5 \implies \min\left\{\mathbb{P}(Y=+1 \mid X), \mathbb{P}(Y=-1 \mid X)\right\} \leq 0.5.
\]
Taking expectations gives \( R(f^*) \leq 0.5 \). Equality holds iff \( X \perp Y \) (\( \mu = 0 \)).
\end{proof}

\begin{proposition}[Information-Theoretic Optimality via Fano]\label{prop:fano}
For any classifier \( f \) with error rate \( \epsilon = R(f) \):
\[
H(Y \mid f(X)) \leq H_b(\epsilon),
\]
where \( H_b(\epsilon) = -\epsilon \log \epsilon - (1-\epsilon)\log(1-\epsilon) \). The Bayes classifier \( f^* \) satisfies:
\[
H(Y \mid f^*(X)) \leq H_b(\epsilon^*) \leq H_b(\epsilon), \quad \forall \epsilon \geq \epsilon^*.
\]
\end{proposition}

\begin{proof}
Apply Fano's inequality directly. The monotonicity of \( H_b(\cdot) \) on \( [0, 0.5] \) ensures:
\[
\epsilon^* \leq \epsilon \implies H_b(\epsilon^*) \leq H_b(\epsilon).
\]
\end{proof}

\subsection{Improving Lion}

We can use the tanh function to improve the optimizer.

Recall the binary classification problem where the output \( Y \in \{+1, -1\} \) and the input \( X \) satisfies the following conditions:
\[
X \mid Y = +1 \sim \mathcal{N}(\mu, \sigma^2), \quad X \mid Y = -1 \sim \mathcal{N}(-\mu, \sigma^2),
\]
with equal priors for both classes:
\[
p(Y = +1) = p(Y = -1) = \frac{1}{2}.
\]

The goal is to find an optimal representation \( T \) of \( X \) in the Information Bottleneck (IB) framework, where the representation \( T \) compresses \( X \) as much as possible while retaining information about \( Y \). Strictly speaking, this problem is formulated as:
\[
\min_{p(t \mid x)} \, I(X; T) - \beta I(T; Y),
\]
where \( \beta > 0 \) is a regularization parameter controlling the trade-off between compression and information retention.

In the IB framework, we aim to maximize \( I(T; Y) \) while minimizing \( I(X; T) \). The optimal representation \( T \) is the one that maximizes the mutual information with \( Y \) while keeping the mutual information with \( X \) small, subject to a constraint.

To find the optimal \( T \), we can use the following:
\[
T = \arg \max_T I(T; Y) \quad \text{subject to} \quad I(X; T) \leq C,
\]
where \( C \) is the compression constraint.

In a binary classification problem, the posterior probability \( p(Y \mid X) \) is a sufficient statistic for \( X \) with respect to \( Y \). Using Bayes' theorem, the posterior probability is given by:
\[
p(Y = +1 \mid X = x) = \frac{p(X = x \mid Y = +1) p(Y = +1)}{p(X = x)}.
\]
Since the conditional distributions are Gaussian, the log-likelihood ratio \( r(x) \) is:
\[
r(x) = \log \frac{p(x \mid Y = +1)}{p(x \mid Y = -1)} = \frac{2\mu}{\sigma^2} x.
\]
Thus, the posterior probability \( p(Y = +1 \mid x) \) becomes:
\[
p(Y = +1 \mid x) = \frac{1}{1 + \exp(-r(x))} = \frac{1}{1 + \exp\left(- \frac{2\mu}{\sigma^2} x \right)}.
\]
We can express the posterior expectation \( E[Y \mid x] \) as:
\[
E[Y \mid x] = 2 p(Y = +1 \mid x) - 1 = \tanh\left(\frac{r(x)}{2}\right) = \tanh\left(\frac{\mu}{\sigma^2} x\right).
\]

In the IB problem, we aim to compress \( X \) while retaining information about \( Y \). According to information theory, for a set of variables satisfying the Markov chain \( Y \to X \to T \), the unique \( T \) that satisfies \( I(T; Y) = I(X; Y) \) (i.e., perfectly preserving the information about \( Y \)) is the sufficient statistic. Therefore, the optimal \( T \) is \( T = \tanh\left( \frac{\mu}{\sigma^2} x \right) \).

We can also derive the $\tanh$ function from a maximal entropy principle. Consider a binary random variable \( Y \in \{+1, -1\} \). Under the constraint \( \mathbb{E}[Y] = \mu \), we derive its maximum entropy distribution.

\subsection*{2. Maximum Entropy Optimization}
Maximize the entropy \( H(p) = -\sum_{y} p(y) \ln p(y) \), subject to:
\begin{itemize}
    \item Normalization: \( p(+1) + p(-1) = 1 \)
    \item Expectation constraint: \( p(+1) - p(-1) = \mu \)
\end{itemize}

Construct the Lagrangian:
\begin{equation*}
    \mathcal{L} = -\sum_{y} p(y) \ln p(y) + x \left( p(+1) - p(-1) - \mu \right) + \lambda \left( p(+1) + p(-1) - 1 \right)
\end{equation*}
where \( x \) and \( \lambda \) are Lagrange multipliers.

Differentiate with respect to \( p(+1) \) and \( p(-1) \):
\begin{align*}
    \frac{\partial \mathcal{L}}{\partial p(+1)} &= -\ln p(+1) - 1 + x + \lambda = 0 \implies \ln p(+1) = x + \lambda - 1 \\
    \frac{\partial \mathcal{L}}{\partial p(-1)} &= -\ln p(-1) - 1 - x + \lambda = 0 \implies \ln p(-1) = -x + \lambda - 1
\end{align*}
Thus:
\[
    p(+1) = e^{x + \lambda - 1}, \quad p(-1) = e^{-x + \lambda - 1}
\]

Using \( p(+1) + p(-1) = 1 \):
\begin{align*}
    &e^{x + \lambda - 1} + e^{-x + \lambda - 1} = 1 \\
    &e^{\lambda - 1}(e^x + e^{-x}) = 1 \\
    &e^{\lambda - 1} = \frac{1}{2\cosh(x)} \quad \text{(since \( e^x + e^{-x} = 2\cosh(x) \))}
\end{align*}

Substitute \( \lambda \):
\[
    p(+1) = \frac{e^x}{2\cosh(x)}, \quad p(-1) = \frac{e^{-x}}{2\cosh(x)}
\]
Unified form:
\[
    p(y; x) = \frac{e^{x y}}{2\cosh(x)}, \quad y \in \{+1, -1\}
\]

Calculate \( \mathbb{E}[Y] = \mu \):
\begin{align*}
    \mu &= \frac{e^x}{2\cosh(x)}(+1) + \frac{e^{-x}}{2\cosh(x)}(-1) \\
        &= \frac{e^x - e^{-x}}{2\cosh(x)} \\
        &= \frac{\sinh(x)}{\cosh(x)} = \tanh(x)
\end{align*}
Thus $\mu = \tanh(x)$.




\section{Conclusion}

In conclusion, this paper explores the information-theoretic foundations that underlie the dynamics of neural network optimizers. By introducing the entropy gap metric, we provide a more holistic view of the optimizer behavior, which sharpness alone cannot fully explain. Our analysis demonstrates that both sharpness and entropy gap play pivotal roles in the optimization process, influencing convergence rates and generalization performance. We also use information-theoretic tools to understand and improve Lion optimizer.

The findings in this paper offer valuable guidance for the future design of optimizers, showing that information-theoretic perspectives can significantly contribute to a better understanding of optimization dynamics and help in the development of more effective and efficient optimization algorithms. This work opens the door for further research on integrating information theory into the development and analysis of optimizers, with the potential to improve both the speed and quality of neural network training.


\section{Related Works}

\paragraph{Information theory for understanding neural networks.}
Information theory has long served as a powerful tool for analyzing the relationship between probability and information. It has been applied to gain insights into the inner workings of neural networks \citep{tishby2000information, tishby2015deep}, aiming to make their operations more interpretable, akin to a "white-box" approach. However, due to the high computational cost of traditional information-theoretic measures, recent research has shifted towards extending these concepts to assess relationships between matrices \citep{bach2022information, skean2023dime, zhang2023relationmatch, zhang2023kernel}. In this context, \citet{tan2023information} utilize matrix mutual information and joint entropy in the field of self-supervised learning.



\clearpage

\bibliography{reference}

\begin{thebibliography}{8}
\providecommand{\natexlab}[1]{#1}
\providecommand{\url}[1]{\texttt{#1}}
\expandafter\ifx\csname urlstyle\endcsname\relax
  \providecommand{\doi}[1]{doi: #1}\else
  \providecommand{\doi}{doi: \begingroup \urlstyle{rm}\Url}\fi

\bibitem[Bach(2022)]{bach2022information}
Francis Bach.
\newblock Information theory with kernel methods.
\newblock \emph{IEEE Transactions on Information Theory}, 2022.

\bibitem[Skean et~al.(2023)Skean, Osorio, Brockmeier, and Giraldo]{skean2023dime}
Oscar Skean, Jhoan Keider~Hoyos Osorio, Austin~J Brockmeier, and Luis Gonzalo~Sanchez Giraldo.
\newblock Dime: Maximizing mutual information by a difference of matrix-based entropies.
\newblock \emph{arXiv preprint arXiv:2301.08164}, 2023.

\bibitem[Tan et~al.(2023)Tan, Yang, Huang, Yuan, and Zhang]{tan2023information}
Zhiquan Tan, Jingqin Yang, Weiran Huang, Yang Yuan, and Yifan Zhang.
\newblock Information flow in self-supervised learning.
\newblock \emph{arXiv preprint arXiv:2309.17281}, 2023.

\bibitem[Tishby \& Zaslavsky(2015)Tishby and Zaslavsky]{tishby2015deep}
Naftali Tishby and Noga Zaslavsky.
\newblock Deep learning and the information bottleneck principle.
\newblock In \emph{2015 ieee information theory workshop (itw)}, pp.\  1--5. IEEE, 2015.

\bibitem[Tishby et~al.(2000)Tishby, Pereira, and Bialek]{tishby2000information}
Naftali Tishby, Fernando~C Pereira, and William Bialek.
\newblock The information bottleneck method.
\newblock \emph{arXiv preprint physics/0004057}, 2000.

\bibitem[Xie et~al.(2022)Xie, Wang, Zhang, Sato, and Sugiyama]{xie2022adaptive}
Zeke Xie, Xinrui Wang, Huishuai Zhang, Issei Sato, and Masashi Sugiyama.
\newblock Adaptive inertia: Disentangling the effects of adaptive learning rate and momentum.
\newblock In \emph{International conference on machine learning}, pp.\  24430--24459. PMLR, 2022.

\bibitem[Zhang et~al.(2023{\natexlab{a}})Zhang, Tan, Yang, Huang, and Yuan]{zhang2023kernel}
Yifan Zhang, Zhiquan Tan, Jingqin Yang, Weiran Huang, and Yang Yuan.
\newblock Matrix information theory for self-supervised learning.
\newblock \emph{arXiv preprint arXiv:2305.17326}, 2023{\natexlab{a}}.

\bibitem[Zhang et~al.(2023{\natexlab{b}})Zhang, Yang, Tan, and Yuan]{zhang2023relationmatch}
Yifan Zhang, Jingqin Yang, Zhiquan Tan, and Yang Yuan.
\newblock Relationmatch: Matching in-batch relationships for semi-supervised learning.
\newblock \emph{arXiv preprint arXiv:2305.10397}, 2023{\natexlab{b}}.

\end{thebibliography}
\bibliographystyle{iclr}


\end{document}